\documentclass[letterpaper]{article}
\usepackage{aaai}
\usepackage{times}
\usepackage{helvet}
\usepackage{courier}
\usepackage{amsthm}
\usepackage{amsmath}
\usepackage{amssymb}
\usepackage{color}
\usepackage{algorithmic}
\usepackage{algorithm}

\usepackage{float}
\usepackage{graphicx}

\newtheorem{lemma}{Lemma}
\newtheorem{theorem}{Theorem}
\newtheorem{definition}{Definition}

\newtheorem{remark}{Remark}

\newcommand{\beqan}{\begin{eqnarray*}}
\newcommand{\eeqan}{\end{eqnarray*}}

\newcommand{\R}{{\mathbb R}}

\newcommand{\E}{\mathbb{E}}

\newcommand{\transpose}{^\mathsf{\scriptscriptstyle T}}

\DeclareMathOperator{\regret}{regret}
\DeclareMathOperator*{\argmin}{arg\,min}
\DeclareMathOperator*{\argmax}{arg\,max}

\newcommand{\cD}{\mathcal{D}}

\newcommand{\cF}{\mathcal{F}}
\newcommand{\cG}{\mathcal{G}}
\newcommand{\cH}{\mathcal{H}}

\newcommand{\cL}{\mathcal{L}}

\newcommand{\cN}{\mathcal{N}}

\newcommand{\cR}{\mathcal{R}}

\newcommand{\cV}{\mathcal{V}}
\newcommand{\cW}{\mathcal{W}}

\newcommand{\bb}{{\bf b}}

\newcommand{\bB}{{\bf B}}

\newcommand{\bI}{{\bf I}}
\newcommand{\bQ}{{\bf Q}}

\newcommand{\bLambda}{{\boldsymbol \Lambda}}
\newcommand{\bmu}{{\boldsymbol \mu}}
\newcommand{\bxi}{{\boldsymbol \xi}}
\newcommand{\bff}{{\boldsymbol f}}

\newcommand{\specialcell}[2][c]{%
  \begin{tabular}[#1]{@{}c@{}}#2\end{tabular}}

\usepackage{natbib}

\begin{document}
%
\title{Spectral Thompson Sampling}
\author{Tom\'a\v s Koc\'ak\\
SequeL team \\
INRIA Lille - Nord Europe\\ 
France\\
\And
Michal Valko\\
SequeL team \\
INRIA Lille - Nord Europe\\ 
France\\
\And
R\'emi Munos \\
SequeL team\\
INRIA Lille, France \\
Microsoft Research NE, USA\\
\And
Shipra Agrawal\\
ML and Optimization Group\\
Microsoft Research\\
Bangalore, India\\
}
\maketitle

\begin{abstract}
\textit{Thompson Sampling} (TS) has  attracted a lot of interest due to its good
empirical
performance, in particular in computational advertising. Though successful,
the tools for its performance analysis appeared only recently.  In this paper,
we describe and analyze SpectralTS algorithm for a bandit problem, where
the payoffs of the choices are \textit{smooth} given an underlying graph. In
this
setting, each choice is a node of a graph and the expected payoffs of the
neighboring nodes are assumed to be similar.  Although the setting has
application both in recommender systems and advertising, the traditional
algorithms would scale poorly with the number of choices. For that purpose we
consider an \emph{effective dimension} $d$, which is small in real-world graphs.
We deliver the analysis showing that the regret of
SpectralTS scales as $d\sqrt{T \ln N}$ with high probability, where $T$ is the time
horizon and $N$ is the number of choices. Since a $d\sqrt{T \ln N}$ regret is
comparable to the known results, SpectralTS offers a computationally more
efficient alternative. We also show that our algorithm is competitive on both
synthetic and real-world data. 
\end{abstract}

\vspace{-.1cm}

\section{Introduction}
Thompson Sampling \citep{thompson1933likelihood}
is one of the oldest heuristics for sequential problems with limited
feedback, also known as \textit{bandit problems}.
It solves the exploration-exploitation
dilemma by a simple and intuitive rule: when choosing
the next action to play, \textit{choose it according
to probability that it is the best one}; that is the one that
\textit{maximizes the expected
payoff}.
Using this heuristic, it is straightforward
to design many bandit algorithms, such as the SpectralTS algorithm
presented in this paper.

What is challenging though, is to provide the analysis and prove
\textit{performance
guarantees} for TS algorithms.
This may be the reason, why TS has not been in the center of interest
of sequential machine learning, where mostly optimistic algorithms were studied
\citep{auer2002finite,auer2002using}.
Nevertheless, the past few years witnessed the rise of interest in TS
due to its empirical performance, in particular in
the computational advertising~\citep{chapelle2011empirical},
a major source of income for Internet companies.
This motivated the researchers to explain the success of TS.

A major breakthrough in this aspect was
the work of~\citet{agrawal2011analysis}, who provided
the first finite-time analysis of TS.
It was shortly after followed by a refined
analysis \citep{kaufmann2012thompson} showing the optimal performance of TS
for Bernoulli distributions.
Some of the asymptotic results for TS were proved by~\citet{may2012optimistic}.
\citet{agrawal2013further} then provided distribution-independent analysis of TS
for the multi-arm bandit. The most relevant results for our work are
by~\citet{agrawal2013thomson}, who bring a new martingale technique, enabling
us to analyze cases where the payoffs of the actions are linear in some basis.
Later,~\citet{korda2013thompson} extended the known optimal TS analysis
to 1-dimensional exponential family distributions.
Finally, in the Bayesian setting, \citet{russo2013eluder,russo2014learning}
analyzed TS with respect to the \textit{Bayesian risk}.

In our prior work~\citep{valko2014spectral}, we introduced a~\textit{spectral
bandit} setting,
relevant for \textit{content-based recommender
systems}~\citep{pazzani2007content},
where the payoff function is expressed as a linear combination of
a \textit{smooth
basis}. In such systems, we aim to recommend a content to a user, based on
her personal preferences. Recommender systems take advantage of the
\textit{content
similarity} in order to offer relevant suggestions. In other words,
we assume that the user preferences are smooth on the similarity graph of 
content items. The results from~\textit{manifold
learning}~\citep{belkin2006manifold} show that the eigenvectors related to the
smallest eigenvalues of the similarity graph Laplacian offer a useful smooth
basis. Another example of leveraging useful structural property present in the
real-world data is to take advantage of the hierarchy of the content
features~\citep{yue2012hierarchical}.

Although LinUCB \citep{li2010contextual}, GP-UCB \citep{srinivas2009gaussian},
and LinearTS
\citep{agrawal2013thomson} could be used for the spectral bandit
setting, they would not scale well with the number of possible items to
recommend. This is why we defined \citep{valko2014spectral} the
\textit{effective dimension} $d$, likely to be small in real-world
graphs, and provided an algorithm based on the optimistic principle:
SpectralUCB.
In this paper, we focus on the TS alternative: \textit{Spectral Thompson
Sampling} (Table~\ref{tab:comparison}). The algorithm is easy to obtain,
since there is no need to derive the upper confidence bounds.
Furthermore, one of the main benefits of SpectralTS is
its computational efficiency.

The main contribution of this paper is the finite-time analysis of
SpectralTS. We prove that the regret of SpectralTS scales as $d\sqrt{T \ln N}$,
which is comparable to the known results.
Although the regret is $\sqrt{\ln N}$ away from
the one of SpectralUCB, this factor is negligible
for the relevant applications, e.g.,~movie recommendation.
 Interestingly, even for the linear case, there is no
polynomial time
algorithm for linear contextual bandits with better than $D \sqrt{T \ln N}$
regret \citep{agrawal2012thompsonarxiv}, where $D$ is the dimension of the
context vector. Optimistic approach
(UCB) for linear contextual bandits is not polynomially implementable, where the
numbers of choices are at least exponential in $D$
(e.g.,~when set of arms is all the vectors in a polytope)
and the approximation given
by~\cite{dani2008price} achieves only $D^{3/2} \sqrt{T}$ regret.
Similarly for the spectral bandit case, SpectralTS offers
a computationally attractive alternative to SpectralUCB
(Table~\ref{tab:comparison}, left).
Instead of computing the upper confidence
bounds for all the arms in each step, we only need to sample from
our current belief of the true model and perform the maximization
given this belief. We support this claim with an empirical evaluation.

\begin{table}
 \vspace*{-1.5em}
\begin{center}
\def\arraystretch{1.9}
 \begin{tabular}[r]{|r|cc|} \hline
 & \textit{Linear} & \textit{Spectral} \\ \hline
\specialcell{\textit{Optimistic Approach} \\[-1.2em] {\scriptsize $D^2N$  per
step
update}}
 & \specialcell{\textbf{LinUCB} \\[-1.2em]  {\scriptsize $D\sqrt{T\ln{T}}$}}
& \specialcell{\textbf{SpectralUCB} \\[-1.2em] {\scriptsize
$d\sqrt{T\ln{T}}$}} \\
 \hline
\specialcell{\textit{Thompson Sampling} \\[-1.2em] {\scriptsize $D^2+DN$ per
step
update}}
 & \specialcell{\textbf{LinearTS} \\[-1.2em]  {\scriptsize $D \sqrt{T \ln N}$}}
& \specialcell{\textbf{\textit{SpectralTS}} \\[-1.2em] {\scriptsize $d\sqrt{T
\ln N}$}} \\
 \hline
\end{tabular}
\caption{Linear vs. Spectral Bandits}
\label{tab:comparison}
 \vspace*{-1.5em}
\end{center}
\end{table}

 \vspace{-.1cm}

\section{Setting}
\label{sec:setting}
In this section, we formally define the \textit{spectral bandit} setting.
The most important notation is summarized in Table~\ref{tab:notation}.
Let $\cG$ be the given graph with the set $\cV$ of $N$ nodes. Let $\cW$ be the
$N\times N$ matrix of edge weights and $\cD$ is a $N\times N$ diagonal matrix
with the entries $d_{ii} = \sum_jw_{ij}$. The graph Laplacian of $\cG$ is
a $N\times N$ matrix defined as $\cL = \cD - \cW$. Let $\cL =
\bQ\bLambda_{\cL}\bQ\transpose$ be the eigendecomposition of $\cL$, where $\bQ$
is $N\times N$ orthogonal matrix with eigenvectors of $\cL$ in columns.
Let $\{\bb_{i}\}_{1\leq i\leq N}$ be the $N$ rows of $\bQ$.
This way, each $\bb_{i}$ corresponds to the features of action $i$
(commonly referred to as the \textit{arm} $i$) in the \textit{spectral basis}.
The reason for spectral basis comes from manifold
learning~\citep{belkin2006manifold}. In our setting, we assume that the
neighboring content has similar payoffs, which means that the payoff
function is smooth on~$\cG$. \citet{belkin2006manifold} showed
that smooth functions can be represented as a linear combination
of eigenvectors with small eigenvalues. This explains the choice
of regularizer in Section~\ref{sec:algo}.

We now describe the learning setting. Each time~$t$, the recommender
chooses an action $a(t)$  and receives a payoff
that is in expectation \textit{linear} in the associated features
$\bb_{a(t)}$,
$$r(t) = \bb_{a(t)}\transpose \bmu + \varepsilon_t,$$
where $\bmu$ encodes the (unknown)
vector of user preferences  and
$\varepsilon_t$ is $R$-sub-Gaussian noise, i.e.,
$$\forall
\xi\in\R,\,\E[e^{\xi\varepsilon_t}\,|\,\{\bb_i\}_{i=1}^N,\cH_{t-1}]\leq
\exp\left(\frac{\xi^2R^2}{2}\right).$$
The history $\cH_{t-1}$ is defined as:
$$\cH_{t-1} = \{a(\tau),\,r(\tau),\,\tau = 1,\dots,t-1\}.$$

We now define the performance metric of any algorithm for the setting above.
The instantaneous (pseudo)-regret in time $t$ is defined as the difference
between the mean payoff (reward) of the optimal arm $a^*$ and the arm $a(t)$
chosen by
the recommender,
$$\regret(t) = \Delta_{a(t)} =
\bb_{a^*}\transpose\bmu-\bb_{a(t)}\transpose\bmu.$$
The performance of any algorithm is measured in terms of cumulative regret,
which is the sum of regrets over time,
$$\cR(T) = \sum_{t = 1}^T\regret(t).$$

 \begin{table}
\vspace*{-1.5em}
\begin{center}
\fbox{
 \parbox{0.75\columnwidth}{
  \begin{itemize}
\vspace*{-0.5em}
\item $N$ -- number of arms
\item $\bb_i$ -- feature vector of arm $i$
\item $d$ -- effective dimension
\item $a(t)$ -- arm played at time $t$
\item $a^*$ -- optimal arm
\item $\lambda$ -- regularization parameter
\item $C$ -- upper bound on $\|\bmu\|_{\bLambda}$
\item $\bmu$ -- true (unknown) vector of weights
\item $v = R\sqrt{6d\ln ((\lambda+T)/(\delta\lambda)}+C$
\item $p = 1/(4e\sqrt{\pi})$
\item $l = R\sqrt{2d\ln ((\lambda+T)T^2/(\delta\lambda))}+C$
\item $g = v\sqrt{4\ln (TN)} + l$
\item $\Delta_i = \bb_{a^*}\transpose\mu -\bb_i\transpose\mu$
\vspace*{-0.5em}
 \end{itemize}
  }}
\end{center}
\vspace*{-0.5em}
 \caption{Overview of the notation.}
\vspace*{-1em}
\label{tab:notation}
 \end{table}
 
 \vspace{-.1cm}

\section{Algorithm}
\label{sec:algo}
In this paper, we use TS to decide which arm to
play. Specifically, we  represent our
current knowledge about $\bmu$ as the normal distribution
$\cN(\hat\bmu(t),v^2\bB_t^{-1})$, where
$\hat\bmu(t)$ is our actual approximation of the unknown parameter $\bmu$ and
$v^2\bB_t^{-1}$ reflects our uncertainty about it.
As mentioned, we assume that the reward function is a linear
combination of eigenvectors of $\cL$ with large coefficients corresponding to
the eigenvectors with small eigenvalues. We encode this assumption into our
initial confidence ellipsoid by setting $\bB_1 = \bLambda = \bLambda_\cL +
\lambda\bI_N$, where $\lambda$ is a regularization parameter.

After that, every time step $t$ we generate a sample
$\tilde\bmu(t)$ from the distribution $\cN(\hat\bmu(t),v^2\bB_t^{-1})$
and choose an arm $a(t)$ that maximizes $\bb_i\transpose\tilde\bmu(t)$.
After receiving a reward, we update our estimate of $\bmu$ and the confidence of
it, i.e.,~we compute $\hat{\bmu}(t+1)$ and $\bB(t+1)$,
\begin{align*}
\bB_{t+1} &= \bB_t + \bb_{a(t)}\bb_{a(t)}\transpose	\\
\hat{\bmu}(t+1) &= \bB_{t+1}^{-1}\left(\sum_{i=1}^{t}\bb_{a(i)}r(i)\right).
\end{align*}

\begin{remark}
Since TS is a Bayesian approach, it requires a prior to run
and we choose it here to be a Gaussian. However, this does not pose any
assumption whatsoever about the actual data
both for the algorithm and the analysis.
The only assumptions we make about the data are: (a) that the mean payoff is
linear in the features, (b) that the noise is sub-Gaussian, and (c) that we know
a bound on the Laplacian norm of the mean reward function.
We provide a frequentist bound on the regret (and not an
average over the prior) which is a much stronger worst case result.
\end{remark}

The computational advantage of SpectralTS in Algorithm~\ref{alg}, compared to
SpectralUCB, is that we do not need to compute the confidence bound for each
arm. Indeed, in SpectralTS we need to sample $\tilde\bmu$
which can be done in $N^2$ time (note that $\bB_t$
is only changing by a rank one update) and a maximum of
$\bb_i\transpose\tilde\bmu$ which can be also done in $N^2$ time.
On the other hand, in SpectralUCB, we need to compute a $\bB_t^{-1}$ norm
for each of $N$ feature vectors which amounts to a $ND^2$ time.
Table~\ref{tab:comparison} (left) summarizes the computational
complexity of the two approaches. Notice that in our setting
$D = N$, which comes to a  $N^2$ vs.~$N^3$ time per step.
We support this argument in Section~\ref{sec:experiments}. Finally
note, that the eigendecomposition needs to be done only once in the beginning 
and since  $\cL$ is diagonally dominant,  this
can be done for $N$ in millions~\citep{koutis2010approaching}.

\begin{algorithm}
   \caption{Spectral Thompson Sampling}
   \label{alg}
\begin{algorithmic}
   \STATE {\bfseries Input:}
   \STATE \quad $N$: number of arms, $T$: number of pulls
   \STATE \quad $\{\bLambda_\cL,\bQ\}$: spectral basis of graph Laplacian $\cL$
   \STATE \quad $\lambda$, $\delta$: regularization and confidence parameters
   \STATE \quad $R$, $C$: upper bounds on noise and $\|\bmu\|_\bLambda$
   \STATE {\bfseries Initialization:}
   \STATE \quad $v = R\sqrt{6d\ln((\lambda + T)/\delta\lambda)}+C$
   \STATE \quad  $\hat{\bmu} =  0_N$, $\bff=0_N$,
   $\bB = \bLambda_\cL + \lambda\bI_N$ 
   \STATE {\bfseries Run:}   
   \FOR{$t=1$ {\bfseries to} $T$}
   \STATE Sample $\tilde{\bmu} \sim \cN(\hat{\bmu},v^2\bB^{-1})$
   \STATE $a(t)\leftarrow \argmax_{a}\bb_a\transpose\tilde{\bmu}$
   \STATE Observe a noisy reward $r(t) = \bb_{a(t)}\transpose\bmu +
\varepsilon_t$
   \STATE $\bff \gets \bff + \bb_{a(t)}r(t)$
   \STATE Update $\bB \gets \bB+ \bb_{a(t)}\bb_{a(t)}\transpose$
   \STATE Update $\hat{\bmu}\gets \bB^{-1}\bff $
   \ENDFOR
\end{algorithmic}
\end{algorithm}

In order to state our result, we first define the \textit{effective dimension},
a quantity shown to be small in the real-world
graphs~\citep{valko2014spectral}.

\begin{definition}
Let the \textbf{effective dimension} $d$ be the largest $d$ such that
$$ (d-1) \lambda_d  \leq \frac{T}{\ln(1 + T /\lambda)}.$$
\end{definition}

We would like to stress that we consider the regime when $T < N$,
because we aim for applications with a large set of arms and we are
interested in a satisfactory performance  after just a few
iterations. For instance, when we aim to recommend $N$ movies, we
would like to have useful recommendations in the time $T < N$, i.e.,~before
the user saw all of them. In the typical $T>N$ setting, $d$ can be of the order
of $N$ and our approach does not bring an improvement over linear bandit
algorithms. The following theorem upper-bounds the cumulative regret of
SpectralTS in terms of $d$.

\begin{theorem}\label{mainTheorem}
Let $d$ be the effective dimension and $\lambda$ be the minimum eigenvalue
of $\bLambda$. If $\|\bmu\|_\bLambda\leq C$ and for all $\bb_i$,
$|\bb_i\transpose \bmu|\leq1$, then the cumulative regret of Spectral
Thompson Sampling is with probability at least $1-\delta$ bounded as
\begin{align*}
\cR(T)\leq\,&\frac{11g}{p}\sqrt{\frac{4+4\lambda}{\lambda}dT\ln\frac{\lambda+T}{\lambda}}+\frac{1}{T}	\\
&+ \frac{g}{p}\left(\frac{11}{\sqrt{\lambda}}+2\right)\sqrt{2T\ln\frac{2}{\delta}},
\end{align*}
where $p = 1/(4e\sqrt{\pi})$ and
\begin{align*}
g =\, &\sqrt{4\ln TN}\left(R\sqrt{6d\ln\left(\frac{\lambda+T}{\delta\lambda}\right)}+C\right)	\\
&+R\sqrt{2d\ln\left(\frac{(\lambda+T)T^2}{\delta\lambda}\right)}+C.
\end{align*}
\end{theorem}

\begin{remark}
Substituting $g$ and $p$ we see that regret bound scales as
$d\sqrt{T\ln{N}}$. Note that $N=D$ could be exponential in $d$ and we need to
consider factor $\sqrt{\ln{N}}$ in our bound. On the other hand, if $N$ is
indeed exponential in $d$, then our algorithm scales with
$\ln{D}\sqrt{T\ln{D}}=\ln(D)^{3/2}\sqrt{T}$ which is even better.
\end{remark}

\vspace{-.1cm}


\section{Analysis}
\label{sec:analysis}

\paragraph{Preliminaries}

In the first five lemmas we state the known results
on which we build in our analysis.

\begin{lemma}
\label{gaussianConcentration}
For a Gaussian distributed random variable $Z$ with mean $m$ and variance
$\sigma^2$, for any $z\geq1$,
$$\frac{1}{2\sqrt{\pi}z}e^{-z^2/2}\leq Pr(|Z-m|>\sigma
z)\leq\frac{1}{\sqrt{\pi}z}e^{-z^2/2}.$$
\end{lemma}

\noindent
Multiple use of Sylvester's determinant theorem gives:
\begin{lemma}
\label{lemma:logdetassum}Let $\bB_t = \bLambda + \sum_{\tau =
1}^{t-1}\bb_{\tau}\bb_\tau\transpose$, then
$$\ln\frac{|\bB_t|}{|\bLambda|} = \sum_{\tau = 1}^{t-1}\ln(1 +
\|\bb_\tau\|_{\bB_{\tau}^{-1}})$$
\end{lemma}

\begin{lemma}{\citep{abbasi2011improved}.}
\label{lem:selfnorm}
 Let $\bB_t = \bLambda + \sum^{t-1}_{\tau=1}\bb_{\tau}\bb_{\tau}\transpose$
and define $\bxi_t = \sum_{\tau=1}^{t-1} \varepsilon_\tau \bb_{\tau}$.
With probability at least $1-\delta$, $\forall t\geq 1:$
\begin{align*}
\|\bxi_t\|^2_{\bB_t^{-1}}
		\leq&  2 R^2 \ln \left(\frac{|\bB_t|^{1/2}}
  {\delta|\bLambda|^{1/2}}\right) .		
\end{align*}
\end{lemma}

\noindent
The next lemma is a generalization of Theorem~2 in~\cite{abbasi2011improved}
for any $\bLambda$.

\begin{lemma}{(Lemma~3 by~\citet{valko2014spectral}).}
\label{lem:confinterval}
 Let $\| \bmu \|_{\bLambda} \leq C$ and $\bB_t$ is as above.
With probability at least $1-\delta$, for any $\bb$ and $t\geq 1:$
\begin{align*}
|\bb\transpose (\hat \bmu(t)-\bmu)|
		\leq& \|\bb  \|_{\bB_t^{-1}} \left(R \sqrt{2\ln \left(\frac{|\bB_t|^{1/2}}
  {\delta|\bLambda|^{1/2}}\right)} + C \right)
\end{align*}
\end{lemma}

\begin{lemma}{(Lemma~7 by~\citet{valko2014spectral}).}
\label{lem:logdetratio}
Let $d$ be the effective dimension and $t\leq T+1$. Then: $$\ln\frac{|\bB_t|}{| \bLambda  |} \leq 2 d \ln\left(1+\frac{T}{\lambda}\right).$$
\end{lemma}



\paragraph{Cumulative Regret analysis}
Our analysis is based on the proof technique of~\citet{agrawal2013thomson}.
The summary of the technique follows. Each time an arm is played, our
algorithm improves the confidence about our actual estimate of 
$\bmu$ via update of $\bB_t$ and thus the update of confidence ellipsoid.
However, when we play a
suboptimal arm, the regret we obtain can be
much
higher than the improvement of our knowledge. To overcome this difficulty, the
arms are divided into two groups of \textit{saturated} and \textit{unsaturated}
arms, based on whether the standard deviation for an arm is smaller than the
 standard deviation of the optimal arm
(Definition~\ref{saturation}) or not. Consequently, the optimal arm is in group
of
unsaturated arms. The idea is to bound the regret of playing
an unsaturated arm in terms of standard deviation and to show that
the probability that the saturated arm is played is small enough.
This way we overcome the difficulty of high regret and small knowledge
 obtained by playing an arm.
In the following we use the notation from Table~\ref{tab:notation}.

\begin{definition}
We define $E^{\hat\mu}(t)$ as the event that for all $i$,
$$|\bb_i\transpose\hat\bmu(t)-\bb_i\transpose\bmu|\leq l\|\bb_i\|_{\bB_t^{-1}}$$
and $E^{\tilde\mu}(t)$ as the event that for all $i$,
$$|\bb_i\transpose\tilde\bmu(t)-\bb_i\transpose\hat\bmu(t)|\leq v\|\bb_i\|_{\bB_t^{-1}}\sqrt{4\ln(TN)}.$$
\end{definition}

\begin{definition}\label{saturation}
We say that an arm $i$ is \textbf{saturated} at time $t$ if
$\Delta_i>g\|\bb_i\|_{\bB_t^{-1}}$, and \textbf{unsaturated}
otherwise (including $a^*$). Let $C(t)$ denote the set of saturated arms at 
time $t$.
\end{definition}

\begin{definition}
We define filtration $\cF_{t-1}$ as the union of the history until time $t-1$
and
features, i.e.,
$$\cF_{t-1} = \{\cH_{t-1}\} \cup \{\bb_i,i=1,\dots,N\}$$
By definition, $\cF_1\subseteq\cF_2\subseteq\dots\subseteq\cF_{T-1}$.
\end{definition}

\begin{lemma}\label{concentrationOfMus}
For all $t$, $0<\delta<1$, $Pr(E^{\hat\mu}(t))\geq1-\delta/T^2$ and for all
possible filtrations $\cF_{t-1}$,
$$Pr(E^{\tilde\mu}(t)\,|\,\cF_{t-1})\geq1-1/T^2.$$
\end{lemma}

\begin{proof}
\textbf{Bounding the probability of event $E^{\hat{\mu}}(t)$:}
Using Lemma \ref{lem:confinterval}, where $C$ is such that
$\|\bmu\|_\bLambda\leq
C$, for all $i$ with probability at least $1-\delta'$ we have
\begin{align*}
|\bb_i\transpose (\hat \bmu(t)-\bmu)| &\leq \|\bb_i  \|_{\bB_t^{-1}} \left(R
\sqrt{2\ln \left(\frac{|\bB_t|^{1/2}} {\delta'|\bLambda|^{1/2}}\right)}
+ C \right)	\\
&= \|\bb_i  \|_{\bB_t^{-1}} \left(R \sqrt{\ln\frac{|\bB_t|} {|\bLambda|}+
2\ln\frac{1}{\delta'}} + C \right).
\end{align*}
Therefore, using Lemma \ref{lem:logdetratio} and substituting $\delta' =
\delta/T^2$, we get that with probability at least $1-\delta/T^2$, for all $i$,
\begin{align*}
|\bb_i\transpose  (\hat \bmu(t)&-\bmu)| 	\\
\leq\, &\|\bb_i  \|_{\bB_t^{-1}} \left(R \sqrt{2d\ln\frac{\lambda+T} {\lambda}+ 2d\ln\frac{T^2}{\delta}} + C \right)	\\
=\, &\|\bb_i  \|_{\bB_t^{-1}} \left(R\sqrt{2d\ln\left(\frac{(\lambda+T)T^2}{\delta\lambda}\right)}+C\right)		\\
=\, &l\|\bb_i  \|_{\bB_t^{-1}}.
\end{align*}
\textbf{Bounding the probability of event $E^{\tilde{\mu}}(t)$:}
The probability of each individual term
$|\bb_i\transpose(\tilde{\bmu}(t)-\hat{\bmu}(t))|<\sqrt{4\ln(TN)}$ can be
bounded using Lemma \ref{gaussianConcentration} to get
\begin{align*}
Pr&\left(|\bb_i\transpose(\tilde{\bmu}(t)-\hat{\bmu}(t))|\geq
v\|\bb_i\|_{\bB^{-1}_t}\sqrt{4\ln(TN)} \right)\\
&\leq \frac{e^{-2\ln{TN}}}{\sqrt{\pi4\ln(TN)}}\leq\frac{1}{T^2N}.
\end{align*}
We complete the proof by taking a union bound over all $N$ vectors~$\bb_i$.
Notice that we took a different approach than~\cite{agrawal2013thomson} to
avoid the dependence on the ambient dimension~$D$.
\end{proof}
\vspace{-0.5em}
\begin{lemma}\label{lemmaWithLambda}
For any filtration $\cF_{t-1}$ such that $E^{\hat\mu}(t)$ is true,
$$Pr\left(\bb_{a^*}\transpose\tilde\bmu(t)>\bb_{a^*}\transpose\bmu \,|\,\cF_{t-1}\right)\geq \frac{1}{4e\sqrt{\pi}}.$$
\end{lemma}

\begin{proof}
Since $\bb_{a^*}\transpose\tilde\bmu(t)$ is a Gaussian random variable with
the mean $\bb_{a^*}\transpose\hat\bmu(t)$ and the standard deviation
$v\|\bb_{a^*}\|_{B_t^{-1}}$, we can use the anti-concentration inequality in
Lemma \ref{gaussianConcentration},
\begin{align*}
Pr&\left(  \bb_{a^*}\transpose\tilde\bmu(t)\geq \bb_{a^*}\transpose\bmu \,|\,
\cF_{t-1}\right)		\\
=\,&Pr\left( \frac{\bb_{a^*}\transpose\tilde\bmu(t)-\bb_{a^*}\transpose\hat\bmu(t)}{v\|\bb_{a^*}\|_{\bB_t^{-1}}}\geq \frac{\bb_{a^*}\transpose\bmu-\bb_{a^*}\transpose\hat\bmu(t)}{v\|\bb_{a^*}\|_{\bB_t^{-1}}} \,|\, \cF_{t-1}\right)	\\
\geq\,&\frac{1}{4\sqrt{\pi}Z_t}e^{-Z_t^2},	\\
\end{align*}
\vskip -3.5em
$$ \qquad \mathrm{where} \quad|Z_t| =
\left|\frac{\bb_{a^*}\transpose\bmu-\bb_{a^*}\transpose\hat\bmu(t)}{v\|\bb_{a^*}
\|_{\bB_t^{-1}}}\right|.$$
Since we consider a filtration $\cF_{t-1}$ such that $E^{\hat\mu}(t)$ is true,
we can upper-bound the numerator to get
\begin{align*}
|Z_t| \leq \frac{l\|\bb_{a^*}\|_{\bB_t^{-1}}}{v\|\bb_{a^*}\|_{\bB_t^{-1}}}=\frac{l}{v}\leq1.
\end{align*}
 \vspace{-2.13em}
Finally,
$$Pr\left(\bb_{a^*}\transpose\tilde\bmu(t)>\bb_{a^*}\transpose\bmu\,|\,\cF_{t-1} \right)\geq \frac{1}{4e\sqrt{\pi}}.$$
\vskip -1.5em
\end{proof}
\begin{lemma}\label{atIsUnsaturated}
For any filtration $\cF_{t-1}$ such that $E^{\hat\mu}(t)$ is true,
$$Pr(a(t)\not\in C(t)\,|\,\cF_{t-1}) \geq \frac{1}{4e\sqrt{\pi}}-\frac{1}{T^2}.$$
\end{lemma}

\begin{proof}
The algorithm chooses the arm with the highest value of
$\bb_i\transpose\tilde{\bmu}(t)$ to be played at time $t$. Therefore if
$\bb_{a^*}\transpose\tilde{\bmu}(t)$ is greater than
$\bb_{j}\transpose\tilde{\bmu}(t)$ for all saturated arms,
i.e.,~$\bb_{a^*}\transpose\tilde{\bmu}(t)>\bb_{j}\transpose\tilde{\bmu}(t),\,
\forall j\in C(t)$, then one of the unsaturated arms (which include the optimal
arm and other suboptimal unsaturated arms) must be played. Therefore,
\begin{align*}
Pr(a(t)\not\in\ & C(t)\,|\,\cF_{t-1})		\\
\geq\, &Pr(\bb_{a^*}\transpose\tilde{\bmu}(t)>\bb_{j}\transpose\tilde{\bmu}(t),\,\forall j\in C(t)\,|\,\cF_{t-1}).
\end{align*}
By definition, for all saturated arms, i.e.,~for all $j\in C(t)$,
$\Delta_j>g\|\bb_{j}\|_{\bB_t^{-1}}$. Now if both of the events
$E^{\hat{\mu}(t)}$ and $E^{\tilde{\mu}(t)}$ are true then, by definition of
these events, for all $j\in C(t)$, $\bb_{j}\transpose\tilde{\bmu}(t)\leq
\bb_{j}\transpose{\bmu}(t)+g\|\bb_{j}\|_{\bB_t^{-1}}$. Therefore, given
the filtration $\cF_{t-1}$, such that $E^{\hat{\mu}}(t)$ is true, either
$E^{\tilde{\mu}}(t)$ is false, or else for all $j\in C(t)$, 
$$\bb_{j}\transpose\tilde{\bmu}(t) \leq \bb_{j}\transpose{\bmu}+g\|\bb_{j}\|_{\bB_t^{-1}} \leq \bb_{a^*}\transpose{\bmu}.$$
Hence, for any $\cF_{t-1}$ such that $E^{\hat{\mu}}(t)$ is true, 
\begin{align*}
Pr(\bb_{a^*}&\transpose\tilde{\bmu}(t)>\bb_{j}\transpose\tilde{\bmu}(t),\,
\forall j\in C(t)\,|\,\cF_{t-1})		\\
\geq\,&Pr(\bb_{a^*}\transpose\tilde{\bmu}(t)>\bb_{a^*}\transpose{\bmu}\,|\,\cF_{t-1})-Pr\left(\overline{E^{\hat{\mu}}(t)}\,|\,\cF_{t-1}\right)	\\
\geq\, &\frac{1}{4e\sqrt{\pi}}-\frac{1}{T^2}.
\end{align*}
In the last inequality we used Lemma~\ref{concentrationOfMus} and
Lemma~\ref{lemmaWithLambda}.
\end{proof}

\begin{lemma}\label{lem:regretBound}
For any filtration $\cF_{t-1}$ such that $E^{\hat\mu}(t)$ is true,
$$\E[\Delta_{a(t)}\,|\,\cF_{t-1}]\leq \frac{11g}{p}\E[\|\bb_{a(t)}\|_{\bB_t^{-1}}\,|\,\cF_{t-1}]+\frac{1}{T^2}$$
\end{lemma}

\begin{proof}
Let $\overline{a}(t)$ denote the unsaturated arm with the smallest norm
$\|\bb_i\|_{\bB_t^{-1}}$, i.e.,
$$\overline{a}(t) = \argmin_{i\not\in C(t)}\|\bb_i\|_{\bB_t^{-1}}.$$
Notice that since $C(t)$ and $\|\bb_i\|_{\bB_t^{-1}}$ for all $i$, are fixed on
fixing $\cF_{t-1}$, so is $\overline{a}(t)$. Now, using Lemma
\ref{atIsUnsaturated}, for any $\cF_{t-1}$ such that $E^{\hat\mu}(t)$ is true,
\begin{align*}
\E[\|\bb_{a(t)}&\|_{\bB_t^{-1}}\,|\,\cF_{t-1}]		\\
\geq\,&\E[\|\bb_{a(t)}\|_{\bB_t^{-1}}\,|\,\cF_{t-1},a(t)\not\in C(t)]	\\
&\cdot Pr(a(t)\not\in C(t)\,|\,\cF_{t-1})		\\
\geq\, &\|\bb_{\overline{a}(t)}\|_{\bB_t^{-1}}\left(\frac{1}{4e\sqrt{\pi}}-\frac{1}{T^2}\right).
\end{align*}
Now, if the events $E^{\hat\mu}(t)$ and $E^{\tilde\mu}(t)$ are true, then for
all $i$, by definition,
$\bb_i\transpose\tilde\bmu(t)\leq\bb_i\transpose\bmu+g\|\bb_i\|_{\bB_t^{-1}}$.
Using this observation along with 
$\bb_{a(t)}\transpose\tilde\bmu(t)\geq\bb_i\transpose\tilde\bmu(t)$ for all~$i$,
\begin{align*}
\Delta_{a(t)}=\,&\Delta_{\overline{a}(t)} + (\bb_{\overline{a}(t)}\transpose\bmu-\bb_{a(t)}\transpose\bmu)	\\
\leq\, &\Delta_{\overline{a}(t)} + (\bb_{\overline{a}(t)}\transpose\tilde\bmu(t)-\bb_{a(t)}\transpose\tilde\bmu(t))	\\
&+g\|\bb_{\overline{a}(t)}\|_{\bB_t^{-1}}+g\|\bb_{a(t)}\|_{\bB_t^{-1}}	\\
\leq\,&\Delta_{\overline{a}(t)} +g\|\bb_{\overline{a}(t)}\|_{\bB_t^{-1}}+g\|\bb_{a(t)}\|_{\bB_t^{-1}}	\\
\leq\,&
g\|\bb_{\overline{a}(t)}\|_{\bB_t^{-1}}+g\|\bb_{\overline{a}(t)}\|_{\bB_t^{-1}}
+g\|\bb_{a(t)}\|_{\bB_t^{-1}}.
\end{align*}
Therefore, for any $\cF_{t-1}$ such that $E^{\hat\mu}(t)$ is true, either
$\Delta_{a(t)}\leq2g\|\bb_{\overline{a}(t)}\|_{\bB_t^{-1}}+g\|\bb_{a(t)}\|_{
\bB_t^{-1}}$, or $E^{\tilde\mu}(t)$ is false. 
We can deduce that
\begin{align*}
\E[\Delta_{a(t)}&\,|\,\cF_{t-1}]	\\
\leq\,&\E\left[2g\|\bb_{\overline{a}(t)}\|_{\bB_t^{-1}}+g\|\bb_{a(t)}\|_{\bB_t^{-1}}\,|\,\cF_{t-1}\right] \\
&+ Pr\left(\overline{E^{\tilde\mu}(t)}\right)	\\
\leq\,&\frac{2g}{p-\frac{1}{T^2}}\E\left[\|\bb_{a(t)}\|_{\bB_t^{-1}}\,|\,\cF_{t-1}\right]	\\
&+g\E\left[\|\bb_{a(t)}\|_{\bB_t^{-1}}\,|\,\cF_{t-1}\right] + \frac{1}{T^2}	\\
\leq\,&\frac{11g}{p}\E[\|\bb_{a(t)}\|_{\bB_t^{-1}}\,|\,\cF_{t-1}]+\frac{1}{T^2}.
\end{align*}
\noindent In the last inequality we used that
$1/(p-1/T^2)\leq 5/p,$	
which holds trivially for $T\leq 4$. For $T\geq 5$, we get that
$T^2\geq5e\sqrt{\pi}$, which holds for $T\geq5.$
\end{proof}

\begin{definition}
We define
$\regret'(t) = \regret(t)\cdot I(E^{\hat{\mu}}(t))$.
\end{definition}

\begin{definition}
A sequence of random variables $(Y_t;\, t\geq0)$ is called a
\textbf{super-martingale}
corresponding to a filtration $\cF_t$, if for all $t$, $Y_t$ is
$\cF_t$-measurable, and for $t\geq1$,
$$\E[Y_t-Y_{t-1}\,|\,\cF_{t-1}] \leq 0.$$
\end{definition}

Next, following~\citet{agrawal2013thomson}, we establish a super-martingale
process that will form the basis of our proof of the high-probability regret
bound.

\begin{definition}
Let
\begin{align*}
X_t &= \regret'(t) - \frac{11g}{p}\|\bb_{a(t)}\|_{\bB_t^{-1}}-\frac{1}{T^2}	\\
Y_t &= \sum_{w = 1}^t X_w.
\end{align*}
\end{definition}

\begin{lemma}
$(Y_t;\, t = 0,\dots,T)$ is a super-martingale process with respect to filtration $\cF_t$.
\end{lemma}

\begin{proof}
We need to prove that for all $t \in \{1,\dots,T\}$, and any possible
$\cF_{t-1}$, $\E[Y_t-Y_{t-1}\,|\,\cF_{t-1}] \leq 0$, i.e.
$$\E[\regret'(t)\,|\,\cF_{t-1}]\leq\frac{11g}{p}\|\bb_{a(t)}\|_{\bB_t^{-1}}+\frac
{1}{T^2}.$$
Note that whether $E^{\hat\mu}(t)$ is true or not, is completely determined by
$\cF_{t-1}$. If $\cF_{t-1}$ is such that $E^{\hat\mu}(t)$ is not true, then
$\regret'(t)=\regret(t)\cdot I(E^{\hat{\mu}}(t))=0$, and the above inequality
holds trivially. Moreover, for $\cF_{t-1}$ such that $E^{\hat\mu}(t)$ holds, the
inequality follows from Lemma \ref{lem:regretBound}.
\end{proof}

Unlike \citep{agrawal2013thomson,abbasi2011improved}, we
do not want to require $\lambda\geq 1$. Therefore, we provide the following
lemma that shows the dependence of $\|\bb_{a(t)}\|_{\bB_t^{-1}}^2$
on $\lambda$.

\begin{lemma}\label{lem:bbound}
For all $t$,
$$\|\bb_{a(t)}\|_{\bB_t^{-1}}^2\leq\left(2+\frac{2}{\lambda}
\right)\ln\left(1+\|\bb_{a(t)}\|_{\bB_t^{-1}}^2\right).$$
\end{lemma}
\begin{proof}
Note, that
$\|\bb_{a(t)}\|_{\bB_t^{-1}}\leq(1/\sqrt{\lambda})\|\bb_{a(t)}\|\leq
(1/\sqrt{\lambda})$ and for all $0\leq x\leq1$ we have
\begin{align}
x\leq2\ln(1 + x).	\label{ineq}
\end{align}
 Now we consider two cases depending on $\lambda$. If
$\lambda\geq1$, we know that $0\leq\|\bb_{a(t)}\|_{\bB_t^{-1}}\leq1$ and
therefore by \eqref{ineq},
$$\|\bb_{a(t)}\|_{\bB_t^{-1}}^2\leq2\ln\left(1+\|\bb_{a(t)}\|_{\bB_t^{-1}}^2\right).$$
Similarly, if $\lambda<1$, then $0\leq\lambda\|\bb_{a(t)}\|^2_{\bB_t^{-1}}\leq 1$
and we get
\begin{align*}
\|\bb_{a(t)}\|_{\bB_t^{-1}}^2&\leq\frac{2}{\lambda}\ln\left(1+\lambda\|\bb_{a(t)}\|_{\bB_t^{-1}}^2\right)	\\
&\leq\frac{2}{\lambda}\ln\left(1+\|\bb_{a(t)}\|_{\bB_t^{-1}}^2\right).
\end{align*}
Combining the two, we get that for all $\lambda\geq0$,
\begin{align*}
\|\bb_{a(t)}\|_{\bB_t^{-1}}^2&\leq\max\left(2,\ \frac{2}{\lambda}\right)\ln\left(1+\|\bb_{a(t)}\|_{\bB_t^{-1}}^2\right)	\\
&\leq\left(2+\frac{2}{\lambda}\right)\ln\left(1+\|\bb_{a(t)}\|_{\bB_t^{-1}}^2\right).
\end{align*}
\vskip -2em
\end{proof}


\begin{proof}[\textbf{Proof of Theorem~\ref{mainTheorem}}]
First, notice that $X_t$ is bounded as  $|X_t|\leq
1+11g/(p\sqrt\lambda)+1/T^2\leq (11/\sqrt\lambda+2)g/p$. Thus, we can apply
Azuma-Hoeffding inequality to obtain that with probability
 at least $1-\delta/2$,
\begin{align*}
\sum_{t=1}^T\regret'(t)\leq&\sum_{t=1}^T\frac{11g}{p}\|\bb_{a(t)}\|_{\bB_t^{-1}}+\sum_{t=1}^T\frac{1}{T^2}	\\
&+\sqrt{2\left(\sum_{t=1}^T\frac{g^2}{p^2}\left(\frac{11}{\sqrt{\lambda}}+2\right)^2\right)\ln\frac{2}{\delta}}.
\end{align*}
Since $p$ and $g$ are constants, then with
probability $1-\delta/2$,
\begin{align*}
\sum_{t=1}^T\regret'(t)\leq&\frac{11g}{p}\sum_{t=1}^T\|\bb_{a(t)}\|_{\bB_t^{-1}}+\frac{1}{T}	\\
&+\frac{g}{p}\left(\frac{11}{\sqrt{\lambda}}+2\right)\sqrt{2T\ln\frac{2}{\delta}}.
\end{align*}


The last step is to upper-bound $\sum_{t=1}^T\|\bb_{a(t)}\|_{\bB_t^{-1}}$. For
this purpose, \citet{agrawal2013thomson} rely on the analysis of \citet{auer2002using} and the assumption that $\lambda\geq1$. We provide an alternative approach
using Cauchy-Schwarz inequality, Lemma~\ref{lemma:logdetassum}, and
Lemma~\ref{lem:bbound} to get
\begin{align*}
\sum_{t=1}^T\|&\bb_{a(t)}\|_{\bB_t^{-1}}\leq\sqrt{T\sum_{t=1}^T\|\bb_{a(t)}\|^2_
{ \bB_t^{-1}}}	\\
&\leq\sqrt{T\left(2+\frac{2}{\lambda}\right)\ln\frac{|\bB_T|}{|\bLambda|}}
\leq\sqrt{\frac{4+4\lambda}{\lambda}dT\ln{\frac{\lambda+T}{\lambda}}}.
\end{align*}

Finally, we know that $E^{\hat\mu}(t)$ holds for all $t$ with probability at
least $1-\frac{\delta}{2}$ and $\regret'(t) = \regret(t)$ for all $t$ with
probability at least $1-\frac{\delta}{2}$. Hence, with probability $1-\delta$,
\begin{align*}
\cR(T)\leq\,&\frac{11g}{p}\sqrt{\frac{4+4\lambda}{\lambda}dT\ln\frac{\lambda+T}{\lambda}}+\frac{1}{T} 	\\
&+ \frac{g}{p}\left(\frac{11}{\sqrt{\lambda}}+2\right)\sqrt{2T\ln\frac{2}{\delta}}.
\end{align*}
\end{proof}

 \vspace{-0.7cm}

\section{Experiments}
\label{sec:experiments}
The aim of this section is to give empirical
evidence that SpectralTS -- a faster and simpler algorithm
than SpectralUCB -- also delivers comparable or better empirical
performance.
For the synthetic experiment, we considered a \textit{Barab\'asi-Albert} (BA)
model~(\citeyear{barabasi1999emergence}), known for its preferential
attachment property, common in real-world graphs.
We generated a random graph using BA model with $N = 250$ nodes and the degree
parameter $3$. For each run, we generated the weights of the edges uniformly
at random. Then we generated $\bmu$, a random vector of weights (unknown
to the algorithms) in order to compute the payoffs and evaluated the cumulative
regret. The $\bmu$ in each simulation was a random linear combination
of the first 3 smoothest eigenvectors of the graph Laplacian.
In all experiments, we had $\delta = 0.001$, $\lambda=1$, and $R =
0.01$. We evaluated the algorithms in $T < N$ regime, where the linear bandit
algorithms are not expected to perform well. Figure~\ref{fig:syn} shows the
results averaged over 10 simulations. Notice that while the result of SpectralTS
is comparable to SpectralUCB, its computational time is much faster due to the
reasons discussed
in Section~\ref{sec:algo}.
Recall that while both algorithms compute the least-square
problem of the same size, SpectralUCB has then to compute the confidence
interval for each arm.

\begin{figure}[ht]
 \begin{center}
\includegraphics[width=0.45\columnwidth]
{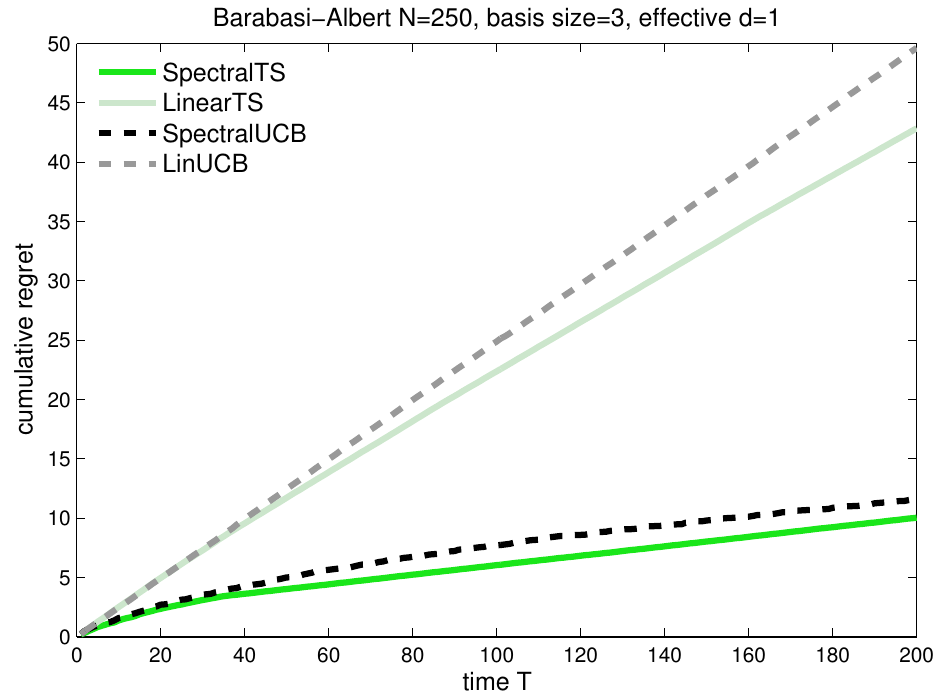}
 \includegraphics[width=0.45\columnwidth]
 {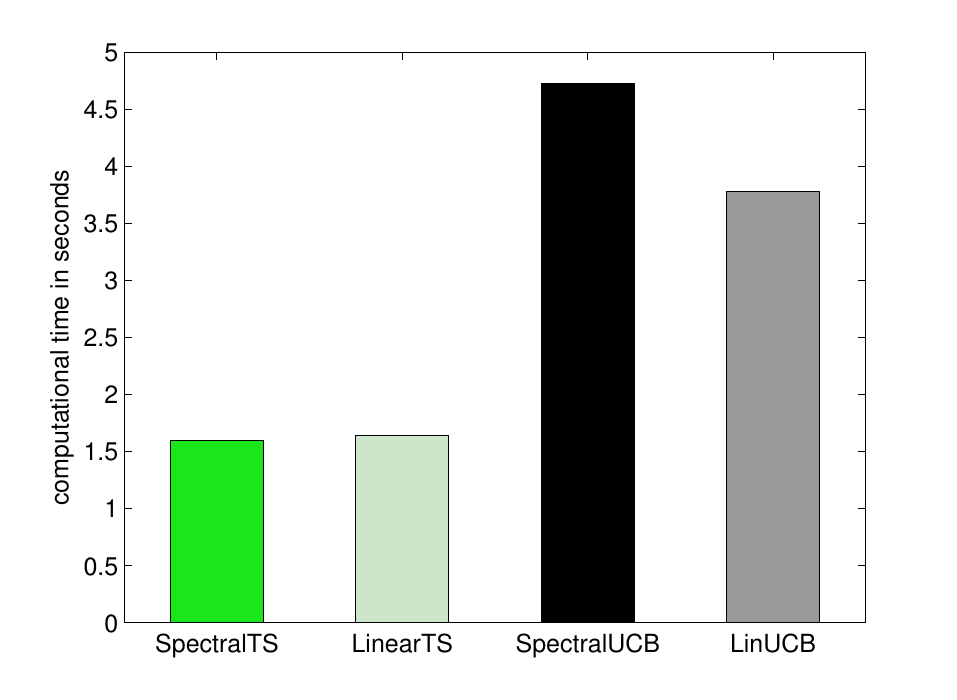}
\vspace{-0.5em}
\caption{Barab\'asi-Albert random graph results}
  \label{fig:syn}
\end{center}
\vspace{-0.25em}
 \end{figure}
\vspace{-.3cm}
Furthermore, we performed the comparison of the algorithms on the
MovieLens dataset \citep{movielens} of the movie ratings.
The graph in this dataset is the graph of 2019 movies with edges corresponding
to the
movie similarities. For each user we have a graph function, unknown to the
algorithms, that assigns to each node, the rating of the particular user.
A detailed description on the preprocessing is deferred
to~\citep{valko2014spectral}.
Our goal is then to recommend the most highly rated content.
Figure~\ref{fig:movavg} shows the MovieLens data results averaged
over 10 randomly sampled users. Notice that the results
follow the same trends as  for the synthetic data.

Our results show that the empirical
performance of the computationally more efficient
SpectralTS is similar or slightly better than the one of SpectralUCB.
The main contribution of this paper is the analysis that
backs up this evidence with a theoretical justification.

%

\begin{figure}[H]
\begin{center}
\includegraphics[width=0.45\columnwidth]
{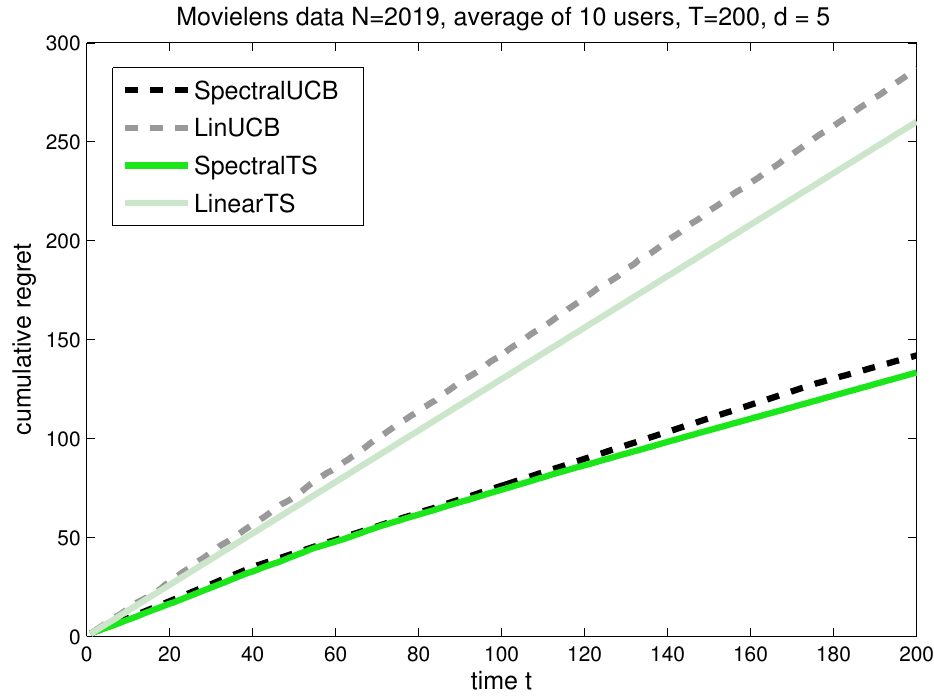}
 \includegraphics[width=0.45\columnwidth]
 {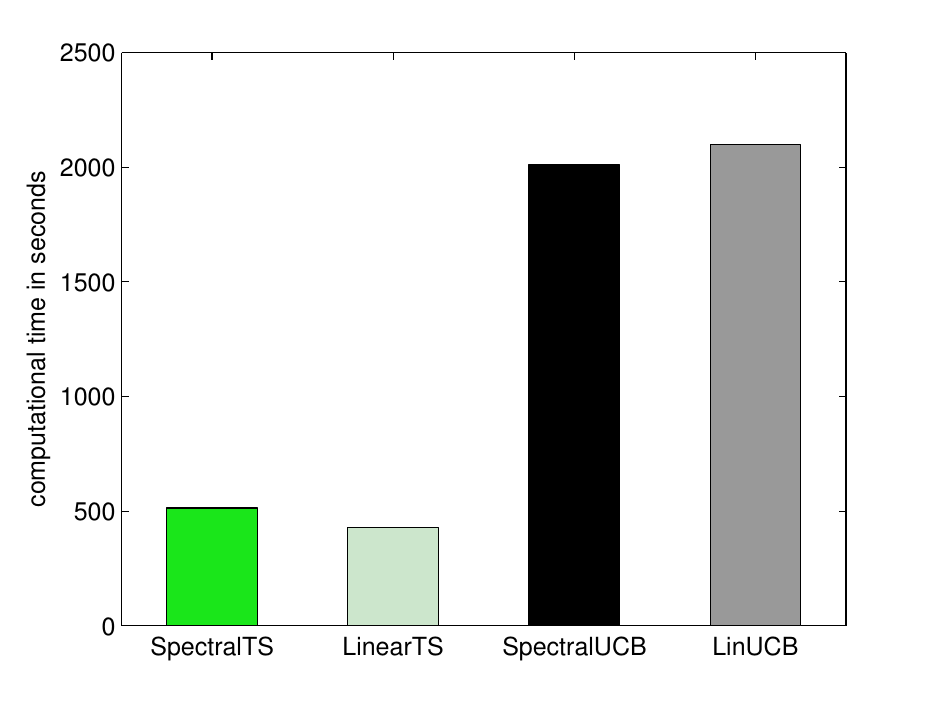}
\vspace{-0.5em}
\caption{MovieLens data results}
  \label{fig:movavg}
\end{center}
 \vspace{-0.25em}
 \end{figure}

%

 \vspace{-.65cm}

\section{Conclusion}
\label{sec:conclusion}

We considered the spectral bandit setting with a reward
function assumed to be smooth on a given graph and proposed 
Spectral Thompson Sampling (TS) for it. Our main contribution is stated in
Theorem~\ref{mainTheorem} where we prove that the regret bound scales
with effective dimension
$d$, typically much smaller than the ambient dimension $D$,
which is the case of linear bandit algorithms.
In our experiments, we showed that
SpectralTS outperforms LinearUCB and LinearTS, and provides similar results
to SpectralUCB in the regime when $T<N$. Moreover, we showed that SpectralTS is
computationally less expensive than SpectralUCB.


\section{Acknowledgements}
\label{sec:Acknowledgements}
The research presented in this paper was supported by French Ministry of
Higher Education and Research, by European Community's
Seventh Framework Programme (FP7/2007-2013) under grant agreement n$^{\rm
o}$270327 (project CompLACS), and by Intel Corporation.


\bibliography{library}
\bibliographystyle{aaai}

\end{document}